\documentclass[a4paper,11pt]{article}
\usepackage[utf8]{inputenc}
\usepackage[margin=1in]{geometry}
\usepackage[numbers]{natbib}

\usepackage{authblk}

\title{Multi-Armed Bandits with Metric Movement Costs}

\author[1]{Tomer Koren}
\author[2]{Roi Livni}
\author[3]{Yishay Mansour}

\affil[1]{Google; \texttt{\small tkoren@google.com}}
\affil[2]{Princeton University; \texttt{\small rlivni@cs.princeton.edu}}
\affil[3]{Tel Aviv University and Google Research; \texttt{\small mansour@cs.tau.ac.il}}

\usepackage{url}            
\usepackage{booktabs}       
\usepackage{amsfonts}       
\usepackage{nicefrac}       
\usepackage{microtype}      

\usepackage[colorlinks,linkcolor=blue,citecolor=blue]{hyperref}

\usepackage{newfloat,framed}
\usepackage[ruled,vlined]{algorithm2e}

\usepackage{xspace}

\usepackage{amsmath,amsfonts,amssymb}
\usepackage{mathtools}
\usepackage{amsthm} 
\usepackage{relsize}

\usepackage[capitalise]{cleveref}
\usepackage{xcolor}
\usepackage{dsfont}
\usepackage{bm}
\usepackage{enumitem}

\usepackage{newfloat,framed}

\DeclareFloatingEnvironment[
    fileext=loa,
    listname=List of Algorithms,
    name=Algorithm,
    placement=h,
]{myalgorithm}
\crefname{myalgorithm}{Algorithm}{Algorithms} 

\newcommand{\wrapalgo}[2][0.9\linewidth]
{%
\begin{center}\setlength{\fboxsep}{5pt}\fbox{\begin{minipage}{#1}
#2
\end{minipage}}\end{center}
\vspace{-0.25cm}
}

\newcommand{\ignore}[1]{}

\theoremstyle{plain}
\newtheorem{theorem}{Theorem}
\newtheorem{lemma}[theorem]{Lemma}
\newtheorem{corollary}[theorem]{Corollary}

\newtheorem*{theorem*}{Theorem}
\newtheorem*{lemma*}{Lemma}
\newtheorem*{corollary*}{Corollary}
\newtheorem*{proposition*}{Proposition}
\newtheorem*{claim*}{Claim}
\newtheorem*{fact*}{Fact}
\newtheorem*{observation*}{Observation}

\theoremstyle{definition}
\newtheorem{definition}[theorem]{Definition}

\newtheorem*{definition*}{Definition}
\newtheorem*{remark*}{Remark}
\newtheorem*{example*}{Example}

 \theoremstyle{plain}
\newtheorem*{theoremaux}{\theoremauxref}
\gdef\theoremauxref{1}

 \newcommand{\thmref}[1]{Theorem~\ref{#1}}

\DeclareMathAlphabet{\mathbfsf}{\encodingdefault}{\sfdefault}{bx}{n}

\let\Pr\relax
\DeclareMathOperator{\Pr}{\mathbb{P}}

\newcommand{\mycases}[4]{{
\left\{
\begin{array}{ll}
    {#1} & {\;\text{#2}} \\[1ex]
    {#3} & {\;\text{#4}}
\end{array}
\right. }}

\newcommand{\lr}[1]{\mathopen{}\left(#1\right)}
\newcommand{\Lr}[1]{\mathopen{}\big(#1\big)}
\newcommand{\LR}[1]{\mathopen{}\Big(#1\Big)}

\newcommand{\lrbra}[1]{\mathopen{}\left[#1\right]}

\newcommand{\set}[1]{\{#1\}}
\newcommand{\lrset}[1]{\mathopen{}\left\{#1\right\}}
\newcommand{\Lrset}[1]{\mathopen{}\big\{#1\big\}}

\newcommand{\abs}[1]{|#1|}

\newcommand{\Lrabs}[1]{\mathopen{}\big|#1\big|}

\newcommand{\indsymb}{\mathds{1}\!}
\newcommand{\ind}[1]{\indsymb\set{#1}}

\newcommand{\distance}{\Delta}

\newcommand{\wt}[1]{{\widetilde{#1}}}

\renewcommand{\O}{O}

\newcommand{\tO}{\wt{\O}}

\newcommand{\E}{\mathbb{E}}
\newcommand{\EE}[2][]{\E_{#1}\lrbra{#2}}

\newcommand{\reals}{\mathbb{R}}
\newcommand{\eps}{\epsilon}
\newcommand{\sig}{\sigma}

\let\oldtfrac\tfrac
\renewcommand{\tfrac}[2]{\smash{\oldtfrac{#1}{#2}}}

\let\nablaold\nabla
\renewcommand{\nabla}{\nablaold\mkern-2mu}

\newcommand{\regret}{\textrm{Regret}}

\newcommand{\mregret}{\textrm{Regret}_\mathsf{MC}}

\newcommand{\klm}{\textbf{EXP3MV}}

\def\expth{\ensuremath{\textsc{Exp3}}\xspace}
\def\klm{\ensuremath{\textsc{SMB}}\xspace}

\setlist{leftmargin=1cm}

\newcommand{\mult}{\!\cdot\!}

\newcommand{\tell}{\wt{\ell}}

\newcommand{\bell}{\bm\bar{\ell}}

\newcommand{\B}{B}
\newcommand{\cB}{\mathcal{B}}
\newcommand{\cC}{\mathcal{C}}
\newcommand{\cD}{\mathcal{D}}

\newcommand{\cT}{\mathcal{T}}

\newcommand{\LCA}{\mathrm{LCA}}
\newcommand{\level}{\mathrm{level}}
\newcommand{\depth}{\mathrm{depth}}

\newcommand{\packing}{N^\mathrm{p}}
\newcommand{\covering}{N^\mathrm{c}}
\newcommand{\upper}{covering~complexity\xspace}
\newcommand{\low}{packing~complexity\xspace}
\newcommand{\upperlow}{complexity\xspace}
\newcommand{\dimu}{\cC_{\smash{\textrm{c}}}}
\newcommand{\diml}{\cC_{\smash{\textrm{p}}}}
\renewcommand{\dim}{\cC}
\newcommand{\good}{well-behaved\xspace}

\def\expth{\ensuremath{\textsc{Exp3}}\xspace}
\def\klm{\ensuremath{\textsc{SMB}}\xspace}

\newcommand{\dist}{\Delta}
\newcommand{\distT}{\Delta_\cT}

\newcommand{\pdim}{\smash{\underline{\cD}}}
\newcommand{\cdim}{\smash{\overline{\cD}}}

\begin{document}

\maketitle

\begin{abstract}%
We consider the non-stochastic Multi-Armed Bandit problem in a setting where there is a
fixed and known metric on the action space that determines a cost for switching
between any pair of actions. The loss of the online learner has two
components: the first is the usual loss of the selected actions, and
the second is an additional loss due to switching between actions.
Our main contribution gives a tight characterization of the
expected minimax regret in this setting, in terms of a complexity
measure~$\dim$ of the underlying metric which depends on its
covering numbers.
In finite metric spaces with $k$ actions, we give an efficient
algorithm that achieves regret of the form
$\smash{\tO(\max\set{\dim^{1/3}T^{2/3},\sqrt{kT}})}$, and show that
this is the best possible. Our regret bound generalizes previous
known regret bounds for some special cases: (i)~the unit-switching
cost regret
$\smash{\wt{\Theta}(\max\set{k^{1/3}T^{2/3},\sqrt{kT}})}$ where
$\dim=\Theta(k)$, and (ii) the interval metric with regret
$\smash{\wt{\Theta}(\max\set{T^{2/3},\sqrt{kT}})}$ where
$\dim=\Theta(1)$. For infinite metrics spaces with Lipschitz loss
functions, we derive a tight regret bound of
$\smash{\wt{\Theta}(T^{\frac{d+1}{d+2}})}$ where $d\ge 1$ is the
Minkowski dimension of the space, which is known to be tight even
when there are no switching costs.
\end{abstract}

\section{Introduction}

Multi-Armed Bandit (MAB) is perhaps one of the most well studied
model for learning that allows to incorporate settings with limited
feedback. In its simplest form, MAB can be thought of as a game
between a learner and an adversary: At first, the adversary chooses
an arbitrary sequence of losses $\ell_1,\ldots,\ell_T$ (possibly
adversarially). Then, at each round the learner chooses an action
$i_t$ from a finite set of actions~$K$. At the end of each round,
the learner gets to observe her loss $\ell_t(i_t)$, and \emph{only}
the loss of her chosen action. The objective of the learner is to
minimize her (external) regret, defined as the expected difference
between her loss, $\smash{\sum_{t=1}^T \ell_t(i_t)}$, and the loss
of the best action in hindsight, i.e., $\smash{\min_{i \in K}
\sum_{t=1}^T \ell_t(i)}$.

One simplification of the MAB is that it assumes that the learner
can switch between actions without any cost, this is in contrast to
online algorithms that maintain a state and have a cost of switching
between states.
One simple intermediate solution is to
add further costs to the learner that penalize \emph{movements
between actions}. (Since we compare the learner to the single best
action, the adversary has no movement and hence no movement cost.)
This approach has been studied in the MAB with unit switching costs
\citep{arora2012online,dekel2014bandits}, where the learner is not
only penalized for her loss but also pays a unit cost for any time
she switches between actions. This simple penalty implicitly
advocates the construction of algorithms that avoid frequent
fluctuation in their decisions.
Regulating switching has been successfully applied to many
interesting instances such as buffering problems
\citep{geulen2010regret}, limited-delay lossy coding
\citep{gyorgy2014near} and dynamic pricing with patient buyers
\citep{feldman2016online}.

The unit switching cost assumes that any pair of actions have
the same cost, which in many scenarios is far from true. For
example, consider an ice-cream vendor on a beach, where his actions
are to select a location and price. Clearly, changing location comes
at a cost, while changing prices might come with no cost. In this
case we can define a interval metric (the coast line) and the
movement cost is the distance. A more involved case is a hot-dog
vendor in Manhattan, which needs to select a location and price.
Again, it makes sense to charge a switching cost between locations
according to their distance, and in this case the Manhattan-distance
seems the most appropriate. Such settings are at the core of our
model for MAB with movement cost.
The authors of \cite{koren2017bandits} considered a MAB problem equipped with an interval metric,
i.e, the actions are $[0,1]$ and the movement cost is the distance
between the actions.
They proposed a new online algorithm, called the
Slowly Moving Bandit~(SMB) algorithm, that achieves optimal regret
bound for this setting, and applied it to a dynamic pricing problem with
patient buyers to achieve a new tight regret bound.

The objective of this paper is to handle general metric spaces, both
finite and infinite.
We show how to generalize the SMB algorithm and its analysis to
design optimal moving-cost algorithms for \emph{any} metric space
over finite decision space.
Our main result identifies an intrinsic complexity measure of the
metric space, which we call the \emph{covering/packing complexity},
and give a tight characterization of the expected movement regret in
terms of the complexity of the underlying metric. In particular, in
finite metric spaces of complexity $\dim$ with $k$ actions, we give
a regret bound of the form
$\smash{\tO(\max\set{\dim^{1/3}T^{2/3},\sqrt{kT}})}$ and present an
efficient algorithm that achieves it. We also give a matching
$\smash{\wt{\Omega}(\max\set{\dim^{1/3}T^{2/3},\sqrt{kT}})}$ lower
bound that applies to \emph{any} metric with complexity $\dim$.

We extend out results to general continuous metric spaces. For such
a settings we clearly have to make some assumption about the losses,
and we make the rather standard assumption that the losses are
Lipchitz with respect to the underlying metric.
In this setting our results depend on a quite different complexity
measures: the upper and lower Minkowski dimensions of the space,
thus exhibiting a phase transition between the finite case (that
corresponds to Minkowski dimension zero) and the infinite case.
Specifically, we give an upper bound on the regret of
$\smash{\wt{\O}(T^{\frac{d+1}{d+2}})}$ where $d \ge 1$ is the
\emph{upper} Minkowski dimension. When the upper and lower Minkowski
dimensions coincide---which is the case in many natural spaces, such
as normed vector spaces---the latter bound matches a lower bound of
\cite{bubeck2011x} that holds even when there are no switching
costs. Thus, a surprising implication of our result is that in
infinite actions spaces (of bounded Minkowski dimension), adding
movement costs do not add to the complexity of the MAB problem!

Our approach extends the techniques of \cite{koren2017bandits} for
the SMB algorithm, which was designed to optimize over an interval
metric, which is equivalent to a complete binary Hierarchally
well-Separated Tree (HST) metric space. By carefully balancing and
regulating its sampling distributions, the SMB algorithm avoids
switching between far-apart nodes in the tree and possibly incurring
large movement costs with respect to the associated metric.
We show that the SMB regret guarantees are much more general than
just binary balanced trees, and give an analysis of the SMB
algorithm when applied to general HSTs. As a second step, we show
that a rich class of trees, on which the SMB algorithm can be
applied, can be used to upper-bound any general metric. Finally, we
reduce the case of an infinite metric space to the finite case via
simple discretization, and show that this reduction gives rise to
the Minkowski dimension as a natural complexity measure.
All of these contractions turn out to be optimal (up to
logarithmic factors), as demonstrated by our matching lower bounds.

\subsection{Related Work}

Perhaps the most well known classical algorithm for non-stochastic
bandit is the \expth Algorithm \citep{auer2002nonstochastic} that
guarantee a regret of $\smash{\tO(\sqrt{kT})}$ without movement
costs. However, for general MAB algorithms there are no guarantees
for slow movement between actions. In fact, it is known that in a
worst case $\smash{\wt{\Omega}(T)}$ switches between actions are expected
(see \cite{dekel2014bandits}).

A simple case of MAB with movement cost is the uniform metric, i.e.,
when the distance between any two actions is the same. This setting
has seen intensive study, both in terms of analyzing optimal regret
rates \citep{arora2012online,dekel2014bandits}, as well as
applications \citep{geulen2010regret, gyorgy2014near,
feldman2016online}. Our main technical tools for achieving lower
bounds is through the lower bound of \citet{dekel2014bandits} that
achieve such bound for this special case.
The general problem of bandits with movement costs has been first
introduced in \cite{koren2017bandits}, where the authors gave an
efficient algorithm for a $2$-HST binary balanced tree metric, as
well as for evenly spaced points on the interval. The main
contribution of this paper is a generalization of these results to
general metric spaces.

There is a vast and vigorous study of MAB in continuous spaces
\citep{Kleinberg04,cope2009regret,auer2007improved,bubeck2011x,yu2011unimodal}.
These works relate the change in the payoff to the change in the
action. Specifically, there has been a vast research on Lipschitz
MAB with stochastic payoffs
\citep{kleinberg2008multi,slivkins2011multi,slivkins2013ranked,kleinberg2010sharp,magureanu2014lipschitz},
where, roughly, the expected reward is Lipschitz.
For applying our results in continuous spaces we too need to assume
Lipschitz losses, however,
our metric defines also the movement cost between actions and not
only relates the losses of similar actions.
Our general findings is that in
Euclidean spaces, one can achieve the same regret bounds when
movement cost is applied. Thus, the SMB algorithm can achieve the
optimal regret rate.

One can model our problem as a deterministic Markov Decision Process
(MDP), where the states are the MAB actions and in every state there
is an action to move the MDP to a given state (which correspond
to switching actions). The payoff would be the payoff of the MAB
action associated with the state plus the movement cost to the next
state. The work of \citet{Ortner10} studies deterministic MDP where
the payoffs are stochastic, and also allows for a fixed uniform
switching cost.
The work of \citet{Even-DarKM09} and it extensions
\citep{NeuGSA14,YuMS2009} studies a MDP where the payoffs are
adversarial but there is full information of the payoffs. Latter
this work was extended to the bandit model by \citet{NeuGSA14}. This
line of works imposes various assumptions regarding the MDP and the
benchmark policies, specifically, that the MDP is ``mixing'' and
that the policies considered has full support stationary
distributions, assumptions that clearly fail in our very specific
setting.

Bayesian MAB, such as in the Gittins index (see \cite{Gittins}),
assume that the payoffs are from some stochastic process.
It is known that when there are switching costs then the existence of
an optimal index policy is not guaranteed \citep{BanksS94}. There
have been some works on special cases with a fixed uniform switching cost
\citep{AgrawalHT1988,AsawaT1996}. The most relevant work is that of
\citet{guha2009multi} which for a general metric over the actions
gives a constant approximation off-line algorithm. For a survey of
switching costs in this context see \cite{Jun2004}.

The MAB problem with movement costs is related to the
literature on online algorithms and the competitive analysis
framework \citep{BorodinEl98}. A prototypical online problem is the
Metrical Task System (MTS) presented by \citet{borodin1992optimal}.
In a metrical task system there are a collection of states and a
metric over the states. Similar to MAB, the online algorithm at each
time step moves to a state, incurs a movement cost according to the
metric, and suffers a loss that corresponds to that state. However,
unlike MAB, in an MTS the online algorithm is given the loss prior
to selecting the new state. Furthermore, competitive analysis has a
much more stringent benchmark: the best sequence of actions in
retrospect.
Like most of the regret minimization literature, we use the best
single action in hindsight as a benchmark, aiming for a vanishing
average regret.

One of our main technical tools is an approximation from above of a
metric via a Metric Tree  (i.e., $2$-HST). $k$-HST metrics have been
vastly studied in the online algorithms starting with
\cite{Bartal96}. The main goal is to derive a simpler metric
representation (using randomized trees) that will both upper and
lower bound the given metric. The main result is to show a bound of
$O(\log n)$ on the expected stretch of any edge, and this is also
the best possible \citep{FakcharoenpholRT04}. It is noteworthy that
for bandit learning, and in contrast with these works, an upper
bound over the metric suffices to achieve optimal regret rate. This
is since in online learning we compete against the best
\emph{static} action in hindsight, which does not move at all and
hence has zero movement cost. In contrast, in a MTS, where one
compete against the best \emph{dynamic} sequence of actions, one
needs both an upper a lower bound on the metric.

\section{Problem Setup and Background}

In this section we recall the setting of Multi-armed Bandit with Movement Costs introduced in \cite{koren2017bandits}, and review the necessary background required to state our main results.

\subsection{Multi-armed Bandits with Movement Costs}

In the Multi-armed Bandits (MAB) with Movement Costs problem, we consider a game between an online learner and an adversary continuing for $T$ rounds.
There is a set $K$, possibly infinite, of actions (or ``arms'') that the learner can choose from.
The set of actions is equipped with a fixed and known metric $\dist$ that determines a cost $\dist(i,j) \in [0,1]$ for moving between any pair of actions $i,j \in K$.

Before the game begins, an adversary fixes a sequence $\ell_1,\ldots,\ell_T : K \mapsto [0,1]$ of loss functions assigning loss values in $[0,1]$ to actions in $K$ (in particular, we assume an oblivious adversary).
Then, on each round $t=1,\ldots,T$, the learner picks an
action $i_t \in K$, possibly at random. At the end of each round $t$, the learner gets to observe her loss (namely, $\ell_t(i_t)$) and nothing else.
In contrast with the standard MAB setting, in addition to the loss $\ell_t(i_t)$ the learner suffers an additional cost due to her movement between actions, which is determined by the metric and is equal to $\dist(i_t,i_{t-1})$.
Thus, the total cost at round $t$ is given by $\ell_t(i_t)+\dist(i_{t-1},i_t)$.

The goal of the learner, over the course of $T$ rounds of the game,
is to minimize her expected movement-regret, which is defined as the
difference between her (expected) total costs and the
total costs of the best fixed action in hindsight (that incurs no
movement costs);
namely, the \emph{movement regret} with respect to a sequence $\ell_{1:T}$ of loss vectors and a metric $\dist$ equals
\begin{align*}
\mregret(\ell_{1:T},\dist)
=
\EE{ \sum_{t=1}^T \ell_t(i_t) + \sum_{t=2}^T
\distance(i_t,i_{t-1})} - \min_{i \in K} \sum_{t=1}^T \ell_t(i)
~.
\end{align*}
Here, the expectation is taken with respect to the learner's
randomization in choosing the actions $i_1,\ldots,i_T$;
notice that, as we assume an oblivious adversary, the loss functions $\ell_t$ are deterministic and cannot depend on the learner's randomization.

\subsection{Basic Definitions in Metric Spaces}

We recall basic notions in metric space that govern the regret in the MAB with movement costs setting.
Throughout we assume a bounded metric space $(K,\dist)$, where for normalization we assume $\dist(i,j) \in [0,1]$ for all $i,j \in K$.
Given a point $i \in K$ we will denote by $\B_\eps(i)=\set{j\in K : \dist(i,j)\le \eps}$ the ball of radius $\eps$ around $i$.

The following definitions are standard.

\begin{definition}[Packing numbers]
A subset $P\subset K$ in a metric space $(K,\dist)$ is an \emph{$\eps$-packing} if the sets $\{\B_\eps(i)\}_{i\in P}$ are disjoint sets.
The \emph{$\eps$-packing number} of $\dist$, denoted $\packing_\eps(\dist)$, is the maximum cardinality of any $\eps$-packing of $K$.
\end{definition}

\begin{definition}[Covering numbers]
A subset $C\subset K$ in a metric space $(K,\dist)$ is an
\emph{$\eps$-covering} if $K \subseteq \cup_{i\in C} \B_\eps(i)$.
The \emph{$\eps$-covering number} of $K$, denoted
$\covering_\eps(\dist)$, is the minimum cardinality of any
$\eps$-covering of $K$.
\end{definition}

\paragraph{Tree metrics and HSTs.}

We recall the notion of a tree metric, and in particular, a metric
induced by an Hierarchically well-Separated (HST) Tree; see
\cite{Bartal96} for more details. Any weighted tree defines a
metric over the vertices, by considering the shortest path between
each two nodes. An HST tree ($2$-HST tree, to be precise) is  a
rooted weighted tree such that: 1) the edge weight from any node to
each of its children is the same and 2) the edge weight along any
path from the root to a leaf are decreasing by a factor $2$ per edge.
We will also assume that all leaves are of the same depth in the
tree (this does not imply that the tree is complete).

Given a tree $\cT$ we let $\depth(\cT)$ denote its height, which is the maximal length of a path from any leaf to the root.
Let $\level(v)$ be the level of a node $v \in \cT$, where the level of the leaves is $0$ and the level of the root is $\depth(\cT)$.
Given nodes $u,v \in \cT$, let $\LCA(u,v)$ be their least common ancestor node in~$\cT$.

The metric which we next define is equivalent (up to a constant factor) to standard tree--metric induced over the leaves by an HST. By a slight abuse of terminology, we will call it HST metric:

\begin{definition}[HST metric]
Let $K$ be a finite set and let $\cT$ be a tree whose leaves are at the same depth and are indexed by elements of $K$.
Then the HST metric $\distT$ over $K$ induced by the tree $\cT$ is defined as follows:
\begin{align*} 
\distT(i,j)
=
\frac{2^{\level(\LCA(i,j))}}{2^{\depth(\cT)}}
\qquad
\forall ~ i,j \in K
.
\end{align*}
\end{definition}

For a HST metric $\distT$, observe that the packing number and
covering number are simple to characterize: for all $0 \le h <
\depth(\cT)$ we have that for $\eps = 2^{h-H}$,
$$
\covering_\eps(\distT) 
= 
\packing_\eps(\distT) 
= 
\Lrabs{\set{v \in \cT : \level(v) = h}}
.
$$

\paragraph{Complexity measures for finite metric spaces.}

We next define the two notions of complexity that, as we will later see, governs the complexity of MAB with metric movement costs.

\begin{definition}[\upper]
The \upper of a metric space $(K,\dist)$ denoted $\dimu(\dist)$ is given by
\[
\dimu(\dist)=\sup_{0<\epsilon<  1} \, \eps \mult \covering_\eps(\dist).
\]
\end{definition}

\begin{definition}[\low]
The \low of a metric space $(K,\dist)$ denoted $\diml(\dist)$ is given by
\[
\diml (\dist)=\sup_{0<\epsilon < 1} \, \eps \mult \packing_\eps(\dist).
\]
\end{definition}

For a HST metric, the two complexity measures coincide as its packing and covering numbers are the same.
Therefore, for a HST metric $\distT$ we will simply denote the complexity of $(K,\distT)$ as $\dim(\cT)$.
In fact, it is known that in any metric space $\packing_\eps(\dist)\le\covering_\eps(\dist)\le \packing_{\smash{\eps/2}}(\dist)$ for all $\eps>0$.
Thus, for a general metric space we obtain that
\begin{align}
\label{eq:equiv}
\diml(\dist)
\le
\dimu(\dist)
\le
2\diml(\dist).
\end{align}

\paragraph{Complexity measures for infinite metric spaces.}

For infinite metric spaces, we require the following definition.

\begin{definition}[Minkowski dimensions]
Let $(K,\dist)$ be a bounded metric space.
The upper Minkowski dimension of $(K,\dist)$, denoted $\cdim(\dist)$, is defined as
\begin{align*}
\cdim(\dist)
=
\limsup_{\eps \to 0} \frac{\log{\packing_\eps(\dist)}}{\log(1/\eps)}
=
\limsup_{\eps \to 0} \frac{\log{\covering_\eps(\dist)}}{\log(1/\eps)}
.
\end{align*}
Similarly, the lower Minkowski dimension is denoted by $\pdim(\dist)$ and is defined as
\begin{align*}
\pdim(\dist)
=
\liminf_{\eps \to 0} \frac{\log{\packing_\eps(\dist)}}{\log(1/\eps)}
=
\liminf_{\eps \to 0} \frac{\log{\covering_\eps(\dist)}}{\log(1/\eps)}
.
\end{align*}
\end{definition}
\ignore{
Note that  if $d = \cdim(\dist)$ then there exists a constant $C>0$ such that $\covering_\eps(\dist) \ge C \eps^{-d}$ for any $\eps>0$;
similarly, if $d' = \pdim(\dist)$ then for some constant $C'>0$ it holds that $\covering_\eps(\dist) \le C' \eps^{-d'}$ for any $\eps>0$.
}
We refer to \cite{tao2009minkowski} for more background on the Minkowski dimensions and related notions in metric spaces theory.

\section{Main Results}

We now state the main results of the paper, which give a complete characterization of the expected regret in the MAB with movement costs problem.

\subsection{Finite Metric Spaces} \label{sec:results-finite}

The following are the main results of the paper.
Detailed proofs are provided in \cref{ap:proofs}.

\begin{theorem}[Upper Bound] \label{thm:upper}
Let $(K,\dist)$ be a finite metric space over $\abs{K} = k$ elements with diameter $\le 1$ and \upper $\dimu = \dimu(\dist)$.
There exists an algorithm such  that for any sequence of loss functions $\ell_1,\ldots,\ell_T$ guarantees that
\begin{align*}
\mregret(\ell_{1:T},\dist)
=
\tO\Lr{ \max\Lrset{ \dimu^{1/3} T^{2/3}, \sqrt{kT} } }
.
\end{align*}
\end{theorem}

\begin{theorem}[Lower Bound] \label{thm:lower}
Let $(K,\dist)$ be a finite metric space over $\abs{K} = k$ elements with diameter $\ge 1$ and \low $\diml = \diml(\dist)$. 
For any algorithm there exists a sequence $\ell_1,\ldots, \ell_{T}$ of loss functions such that
\begin{align*}
\mregret(\ell_{1:T},\dist)
=
\wt{\Omega}\Lr{ \max\Lrset{ \diml^{1/3} T^{2/3}, \sqrt{kT} } }
.
\end{align*}
\end{theorem}

Recalling \cref{eq:equiv}, we see that the regret bounds obtained in \cref{thm:upper,thm:lower} are matching up to logarithmic factors.
Notice that the tightness is achieved \emph{per instance}; namely, for any given metric we are able to fully characterize the regret's rate of growth as a function of the intrinsic properties of the metric. (In particular, this is substantially stronger than demonstrating a specific metric for which the upper bound cannot be improved.)
Note that for the lower bound statement in \cref{thm:lower} we require that the diameter of $K$ is bounded away from zero, where for simplicity we assume a constant bound of $1$. Such an assumption is necessary to avoid degenerate metrics. Indeed, when the diameter is very small, the problem reduces to the standard MAB setting without any additional costs and we obtain a regret rate of $\Omega(\sqrt{kT})$.

Notice how the above results extend known instances of the problem from previous work: for uniform movement costs (i.e., unit switching costs) over $K=\set{1,\ldots,k}$ we have $\dimu=\Theta(k)$, so that the obtain bound is
$\smash{\wt{\Theta}(\max\set{k^{1/3} T^{2/3}, \sqrt{kT}})},$
which recovers the results in \cite{arora2012online,dekel2014bandits};
and for a $2$-HST binary balanced tree with $k$ leaves, we have $\dimu=\Theta(1)$ and the resulting bound is
$\smash{\wt{\Theta}(\max\set{T^{2/3}, \sqrt{kT}})},$
which is identical to the bound proved in \cite{koren2017bandits}.

The $2$-HST regret bound in \cite{koren2017bandits} was primarily used to obtain regret bounds for the action space $K=[0,1]$. 
In the next section we show how this technique is extended for infinite metric space to obtain regret bounds that depend on the dimensionality of the action space.

\subsection{Infinite Metric Spaces} \label{sec:results-infinite}

When $(K,\dist)$ is an infinite metric space, without additional constraints on the loss functions, the problem becomes ill-posed with a linear regret rate, even without movement costs.
Therefore, one has to make additional assumptions on the loss functions in order to achieve sublinear regret.
One natural assumption, which is common in previous work, is to assume that the loss functions $\ell_1,\ldots,\ell_T$ are all $1$-Lipschitz with respect to the metric $\dist$.
Under this assumption, we have the following result.

\begin{theorem} \label{thm:coverdim}
Let $(K,\dist)$ be a metric space with diameter $\le 1$ and upper Minkowski dimension $d = \cdim(\dist)$, such that $d\ge 1$.
There exists a strategy such that for any sequence of loss functions $\ell_1,\ldots,\ell_T$, which are all $1$-Lipschitz with respect to $\dist$, guarantees that
\begin{align*}
\mregret(\ell_{1:T},\dist)
=
\tO\Lr{ T^\frac{d+1}{d+2} }
.
\end{align*}
\end{theorem}

Again, we observe that the above result extend the case of $K=[0,1]$ where $d=1$. 
Indeed, for Lipschitz functions over the interval a tight regret bound of $\smash{\wt\Theta(T^{2/3})}$ was achieved in \cite{koren2017bandits}, which is exactly the bound we obtain above.

We mention that a lower bound of $\smash{\wt\Omega(T^{\smash{\frac{d+1}{d+2}}})}$ is known for MAB in metric spaces with Lipschitz cost functions---even \emph{without movement costs}---where $d = \pdim(\dist)$ is the lower Minkowski dimension.

\begin{theorem}[\citet{bubeck2011x}] \label{thm:coverdimlow}
Let $(K,\dist)$ be a metric space with diameter $\le 1$ and lower Minkowski dimension $d = \pdim(\dist)$, such that $d\ge 1$.
Then for any learning algorithm, there exists a sequence of loss function $\ell_1,\ldots,\ell_T$, which are all $1$-Lipschitz with respect to $\dist$, such that the regret (without movement costs) is
$
\smash{\wt{\Omega}\Lr{ T^\frac{d+1}{d+2} }}
.
$
\end{theorem}

In many natural metric spaces in which the upper and lower Minkowski
dimensions coincide (e.g., normed spaces), the bound of
\cref{thm:coverdim} is tight up to logarithmic factors in $T$. In
particular, and quite surprisingly, we see that the movement costs
do not add to the regret of the problem!

It is important to note that \cref{thm:coverdim} holds only for
metric spaces whose (upper) Minkowski dimension is at least $1$.
Indeed, finite metric spaces are of Minkowski dimension zero, and as
we demonstrated in \cref{sec:results-finite} above, a
$\smash{O(\sqrt{T})}$ regret bound is not achievable. Finite matric
spaces are associated with a complexity measure which is very
different from the Minkowski dimension (i.e., the covering/packing
complexity). In other words, we exhibit a phase transition between
dimension $d=0$ and $d\ge1$ in the rate of growth of the regret
induced by the metric.

\section{Algorithms}

In this section we turn to prove \cref{thm:upper}. Our strategy is much inspired by the approach in \cite{koren2017bandits}, and we employ a two-step approach: First, we consider the case that the metric is a HST metric;
we then turn to deal with general metrics, and show how to upper-bound any metric with a HST metric.

\subsection{Tree Metrics: The Slowly-Moving Bandit Algorithm}

In this section we analyze the simplest case of the problem, in which the metric $\dist=\distT$ is induced by a HST tree $\cT$ (whose leaves are associated with actions in $K$).
In this case, our main tool is the Slowly-Moving Bandit (SMB) algorithm  \cite{koren2017bandits}: we demonstrate how it can be applied to general tree metrics, and analyze its performance in terms of intrinsic properties of the metric.

We begin by reviewing the SMB algorithm.
In order to present the algorithm we require few additional notations.
The algorithm receives as input a tree structure over the set of actions $K$, and its operation depends on the tree structure.
We fix a HST tree $\cT$ and let $H = \depth(\cT)$.
For any level $0 \le h \le H$ and action $i \in K$,
let $A_h(i)$ be the set of leaves of $\cT$ that share a common ancestor with $i$ at level $h$ (recall that level $h=0$ is the bottom--most level corresponding to the singletons). In terms of the tree metric we have that $A_{h}(i)=\{j: \distT(i,j) \le 2^{-H+h}\}$.

The \klm algorithm is presented in \cref{alg:smb}. The algorithm is
based on the multiplicative update method, in the spirit of
\textsc{Exp3} algorithms \cite{auer2002nonstochastic}. Similarly to
\textsc{Exp3}, the algorithm computes at each round $t$ an estimator
$\tell_t$ to the loss vector $\ell_t$ using the single loss value
$\ell_t(i_t)$ observed. In addition to being an (almost) unbiased
estimate for the true loss vector, the estimator $\tell_t$ used by
\klm has the additional property of inducing slowly-changing
sampling distributions $p_t$: This is done by choosing at random a
level $h_t$ of the tree to be rebalanced (in terms of the weights maintained by the algorithm): As a result, the marginal probabilities $p_{t+1}(A_{h_t}(i))$ are not changed at round $t$.

In turn, and in contrast with \textsc{Exp3}, the algorithm choice of
action at round $t+1$ is not purely sampled from $p_t$, but rather conditioned on our last choice of level $h_t$. This is informally justified by the fact that $p_t$ and $p_{t+1}$ agree on the marginal distribution of $A_{h_t}(i_t)$, hence we can think of the level drawn at round $t$ as if it were drawn subject to $p_{t+1}(A_{h_t})=p_{t}(A_{h_t})$.
\begin{myalgorithm}[h]
\wrapalgo[0.70\textwidth]{
Input: A tree $\cT$ with a set of finite leaves $K$, $\eta>0$.\\
Initialize: $H=\depth(\cT)$, $A_{h} (i) = B_{2^{-H+h}}(i), ~ \forall i\in K, 0\le h\le H$ \\
Initialize $p_1 = \text{unif}(K)$, $h_0 = H$ and $i_0 \sim p_1$\\
For $t=1,\ldots,T$:
\begin{enumerate}[nosep,label=(\arabic*)]
\item
Choose action $i_t \sim p_t(\,\cdot \mid A_{h_{t-1}}(i_{t-1}))$, observe loss $\ell_t(i_t)$
\item
Choose $\sig_{t,0},\ldots,\sig_{t,H-1} \in \set{\pm 1}$ uniformly at random;\\
let $h_t = \min\set{0 \le h \le H : \sig_{t,h} < 0}$ where $\sig_{t,H} = -1$
\item
Compute vectors $\bell_{t,0},\ldots,\bell_{t,H-1}$ recursively via
\begin{flalign*} 
\bell_{t,0}(i)
=
\frac{\ind{i_t = i}}{p_t(i)} \ell_t(i_t)
,
&&
\end{flalign*}
and for all $h \ge 1$:
\begin{flalign*}
\bell_{t,h}(i)
=
-\frac{1}{\eta} \ln\lr{ \sum_{j \in A_{h}(i)} \frac{p_t(j)}{p_t(A_{h}(i))} e^{ -\eta (1+\sig_{t,h-1}) \bell_{t,h-1}(j) } }
&&
\end{flalign*}
\item
Define $E_t = \set{i : \text{$p_t(A_h(i)) < 2^h\eta$ for some $0 \le h < H$}}$ and set:
\begin{align*}
\tell_t
=
\mycases
    {0}{if $i_t \in E_t$;}
    {\bell_{t,0} + \sum_{h=0}^{H-1} \sig_{t,h} \bell_{t,h}}{otherwise}
\end{align*}
\item
Update:
\vspace{-10pt}
\begin{align*}
p_{t+1}(i) = \frac{ p_t(i) \, e^{-\eta\tell_t(i)} }{ \sum_{j=1}^k
p_t(j) \, e^{-\eta \tell_t(j)} } \qquad \forall ~ i \in K
\end{align*}
\end{enumerate}
}
\caption{The \klm algorithm.} \label{alg:smb}
\end{myalgorithm}


A key observation is that by directly applying SMB to the metric $\distT$, we can achieve the following regret bound:

\begin{theorem} \label{lem:smb-tree}
Let $(K,\distT)$ be a metric space defined by a $2$-HST $\cT$ with $\depth(\cT) = H$ and \upperlow $\dim(\cT)=\dim$.
Using SMB algorithm we can achieve the following regret bound:
\begin{align} \label{eq:smb}
\mregret(\ell_{1:T},\distT)
=
\O\LR{ H \sqrt{2^H T \dim \smash{\log\dim }} + H 2^{-H} T }
~.
\end{align}
\end{theorem}

To show \cref{lem:smb-tree}, we adapt the analysis of \cite{koren2017bandits} (that applies only to complete binary HSTs) to handle more general HSTs.
We defer this part of our analysis to the appendix, since it follows from a technical modification of the original proof; for the proof of \cref{lem:smb-tree}, see \cref{sec:smb-analysis}.

For a tree that is either too deep or too shallow,
\cref{eq:smb} may not necessarily lead to a sublinear regret bound,
let alone optimal. The main idea behind achieving optimal regret
bound for a general tree, is to modify it until one of two things
happen: Either we have optimized the depth so that the two terms in
the left-hand side of \cref{eq:smb} are of same order: In that case,
we will show that one can achieve regret rate of order
$O(\dim(\cT)^{1/3}T^{2/3})$. If we fail to do that, we show that the
first term in the left-hand side is the dominant one, and it will be
of order $\O(\sqrt{kT})$.

For trees that are in some sense ``well behaved" we have the
following Corollary of \cref{lem:smb-tree} (for a proof see
\cref{prf:smb-goodtree}).
\begin{corollary} \label{cor:smb-goodtree}
Let $(K,\distT)$ be a metric space defined by a tree $\cT$ over $\abs{K}=k$ leaves with $\depth(\cT)=H$ and \upperlow $\dim(\cT)=\dim$.
Assume that $\cT$ satisfies the following:
\begin{enumerate}[nosep,label=(\arabic*)]
\item\label{it:01} $ 2^{-H} H T\le \sqrt{2^H H \dim T}$;
\item\label{it:01.5} One of the following is true:
\begin{enumerate}[nosep]
\item\label{it:02} $2^{H}\dim \le k$;
\item\label{it:03} $ 2^{-(H-1)} (H-1) T\ge \sqrt{2^{H-1} (H-1) \dim T}$.
\end{enumerate}
\end{enumerate}
Then, the SMB algorithm can be used to attain 
$ 
\mregret(\ell_{1:T},\distT)
=
\tO\Lr{ \max\Lrset{ \dim^{1/3} T^{2/3}, \sqrt{kT} } }
.
$
\end{corollary}


The following establishes \thmref{thm:upper} for the special case
of tree metrics (see \cref{prf:tree} for proofs).

\begin{lemma}\label{lem:tree}
For any tree $\cT$ and time horizon $T$, there exists a tree $\cT'$ (over the same set $K$ of $k$ leaves) that satisfies the conditions of \cref{cor:smb-goodtree}, such that $\dist_{\cT'} \ge \dist_{\cT}$ and $\dim(\cT') =  \dim(\cT)$.
Furthermore, $\cT'$ can be constructed efficiently from $\cT$ (i.e., in time polynomial in $\abs{K}$ and $T$).
Hence, applying SMB to the metric space $(K,\dist_{\cT'})$ leads to
$
\mregret(\ell_{1:T},\distT)
=
\tO\Lr{ \max\Lrset{\dim(\cT)^{1/3} T^{2/3}, \sqrt{kT}} }
.
$
\end{lemma}


\subsection{General Finite Metrics}

Finally, we obtain the general finite case as a corollary of the following. 

\begin{lemma}\label{lem:main2}
Let $(K,\dist)$ be a finite metric space.
There exists a tree metric $\distT$ over $K$ (with $\abs{K}=k$) such that $4\distT$, dominates $\dist$ (i.e., such that $4\distT(i,j) \ge \dist(i,j)$ for all $i,j \in K$) for which $\dim(\cT) = O(\dimu(\dist)\log{k})$.
Furthermore, $\cT$ can be constructed efficiently. 
\end{lemma}

\begin{proof}
Let $H$ be such that the minimal distance in $\Delta$ is larger than $2^{-H}$. For each $r=2^{-1},2^{-2},\ldots, 2^{-H}$ we let $\{\B_r(i_{\{1,r\}}),\ldots, \B_r(i_{\{m_r,r\}})\} = \cB_{r}$ be a covering of $K$ of size $\covering_{r} (\cT)\log{k}$ using balls of radius~$r$. Note that finding a minimal set of balls of radius~$r$ that covers $K$ is exactly the set cover problem. Hence, we can efficiently approximate it (to within a $\O(\log{k})$ factor) and construct the sets $\mathcal{B}_{r}$.

We now construct a tree graph, whose nodes are associated with the cover balls: The leaves correspond to singleton balls, hence correspond to the action space.
For each leaf $i$ we find an action $a_1(i) \in K$ such that:
$
i \in \B_{2^{-H+1}}(a_1(i)) \in \mathcal{B}_{2^{-H+1}}.
$
If there is more than one, we arbitrarily choose one, and we connect an edge between $i$ and $\B_{2^{-H+1}}(a_1(i))$.
We continue in this manner inductively to define $a_{r}(i)$ for every $a$ and $r<1$: given $a_{r-1}(i)$ we find an action $a_r(i)$ such that
$
a_{r-1}(i)\in \B_{2^{-H+r}}(a_r(i))\in \mathcal{B}_{2^{-H+r}},
$
and we connect an edge from $\B_{2^{-H+r-1}}(a_{r-1}(i))$ and  $\B_{2^{-H+r}}(a_{r}(i))$.

We now claim that the metric induced by the tree graph dominates up to factor $4$ the original metric. Let $i,j\in K$ such that $\distT(i,j)< 2^{-H+r}$ then by construction there are
$i,a_1(i),a_2(i),\ldots a_r(i)$ and $j,a_1(j),a_2(j),\ldots a_r(j)$, such that $a_r(i)=a_r(j)$ and for which it holds that $\dist(a_s(i),a_{s-1}(i))\le 2^{-H+s}$ and similarly $\dist(a_s(j),a_{s-1}(j))\le 2^{-H+s}$ for every $s\le r$. Denoting $a_0(i)=i$ and $a_0(j)=j$, we have that
\begin{align*}
\dist(i,j)
&\le
\sum_{s=1}^r \dist(a_{s-1}(i),a_s(i))+\sum_{s=1}^r \dist(a_{s-1}(j),a_s(j))
\\
&\le
2\sum_{s=1}^{r} 2^{-H+s}
\le
2 \mult 2^{-H} \mult 2^{r+1}
\le
4\distT(i,j)
.&&\qedhere
\end{align*}
\end{proof}

\subsection{Infinite Metric Spaces}

Finally, we address infinite spaces by discretizing the space $K$ and reducing to the finite case.
Recall that in this case we also assume that the loss functions are Lipschitz.

\begin{proof}[Proof of \cref{thm:coverdim}]
Given the definition of the covering dimension $d = \cdim(\dist) \ge 1$, it is straightforward that for some constant $C>0$ (that might depend on the metric $\dist$) it holds that $\covering_r(\dist) \le C r^{-d}$ for all $r>0$.
Fix some $\eps>0$, and take a minimal $2\eps$-covering $K'$ of $K$ of size $\abs{K'} \le C (2\eps)^{-d} \le C \eps^{-d}$.
Observe that by restricting the algorithm to pick actions from $K'$, we might lose at most $\O(\eps T)$ in the regret.
Also, since $K'$ is minimal, the distance between any two elements in $K'$ is at least $\eps$, thus the \upper of the space has
\begin{align*}
\dimu(\dist)
=
\sup_{r \ge \eps} \, r \mult \covering_r(\dist)
\le
C \sup_{r \ge \eps} \, r^{-d+1}
\le
C \eps^{-d+1}
,
\end{align*}
as we assume that $d \ge 1$.
Hence, by \cref{thm:upper} and the Lipschitz assumption, there exists an algorithm for which
\begin{align*}
\mregret(\ell_{1:T},\dist)
=
\tO\lr{ \max\Lrset{ \eps^{-\frac{d-1}{3}} T^{\frac{2}{3}}, \eps^{-\frac{d}{2}} T^{\frac{1}{2}}, \eps T } }
.
\end{align*}
A simple computation reveals that $\eps=\Theta(T^{-\frac{1}{d+2}})$ optimizes the above bound, and leads to $\tO(T^\frac{d+1}{d+2})$ movement regret.
\end{proof}

\section*{Acknowledgements} 

RL is supported in funds by the Eric and Wendy Schmidt Foundation for strategic innovations.
YM is supported in part by a grant from the Israel Science
Foundation, a grant from the United States-Israel Binational Science
Foundation (BSF), and the Israeli Centers of Research Excellence
(I-CORE) program (Center No. 4/11).

\bibliographystyle{abbrvnat}
\bibliography{treeband}

\appendix

\section{Proofs}\label{ap:proofs}

\subsection{Proof of \cref{cor:smb-goodtree}}\label{prf:smb-goodtree}

\begin{corollary*}
Let $(K,\distT)$ be a metric space defined by a tree $\cT$ over $\abs{K}=k$ leaves with $\depth(\cT)=H$ and \upperlow $\dim(\cT)=\dim$.
Assume that $\cT$ satisfies the following:
\begin{enumerate}[nosep,label=(\arabic*)]
\item\label{it:1} $ 2^{-H}  T\le \sqrt{2^H  \dim T}$;
\item\label{it:1.5} One of the following is true:
\begin{enumerate}[nosep]
\item\label{it:2} $2^{H-1}\dim \le k$;
\item\label{it:3} $ 2^{-(H-1)} T\ge \sqrt{2^{H-1} \dim T}$.
\end{enumerate}
\end{enumerate}
Then, the SMB algorithm can be used to attain regret bounded as
\begin{align*} 
\mregret(\ell_{1:T},\distT)
=
\tO\LR{ \max\Lrset{ \dim^{1/3} T^{2/3}, \sqrt{kT} } }
~.
\end{align*}
\end{corollary*}
\begin{proof}
\crefname{enumi}{condition}{conditions} 
Notice that by  \cref{it:1} and \cref{lem:smb-tree} we have
$
\mregret
\le
c \sqrt{T 2^H H^2 \dim \log \dim }
$
for some constant factor~$c$.
First assume, we have that \cref{it:2} holds. Note that in particular $\max(2^{H-1},\dim)\le k$. We thus, obtain:
\begin{align}
\mregret & \le c\sqrt{T2^H H^2 \dim \log \dim }
\nonumber \\
& \le c\sqrt{T 2^H \dim\log^2 k  \log k } & \because~\max(2^{H-1}, \dim)\le k
\nonumber \\ \label{eq:1}
&\le 4c \sqrt{T k \log \dim}\log^{3/2}k \le 4c \sqrt{k T}\log^{3/2}{k} .
\end{align}
The next case is that \cref{it:3} holds. By reordering we can write \cref{it:3} as:
\begin{align} \label{eq:2}
2^{H} &\le 8 T^{1/3} \dim^{-1/3} ;
\end{align}
Which also implies $H=O(\log T)$, hence:
\begin{align}\label{eq:3}
\sqrt{2^H H^2 T \dim\log \dim}
&=
\tO\Lr{\dim^{1/3}T^{2/3} }
.
\end{align}
Overall, we see that in both cases the regret is bounded by the maximum between the two terms in \cref{eq:1,eq:3}.
%
\end{proof}

\subsection{Proof of \cref{lem:tree}} \label{prf:tree}

\begin{lemma*}
For any tree $\cT$ and time horizon $T$, there exists a tree $\cT'$ (over the same set $K$ of $k$ leaves) that satisfies the conditions of \cref{cor:smb-goodtree}, such that $\dist_{\cT'} \ge \dist_{\cT}$ and $\dim(\cT') =  \dim(\cT)$.
Furthermore, $\cT'$ can be constructed efficiently from $\cT$ (i.e., in time polynomial in $\abs{K}$ and $T$).
Hence, applying SMB to the metric space $(K,\dist_{\cT'})$ leads to
\[
\mregret(\ell_{1:T},\distT)
=
\tO\lr{ \max\Lrset{\dim(\cT)^{1/3} T^{2/3}, \sqrt{kT}} }
.
\]
\end{lemma*}

\begin{proof}
\crefname{enumi}{condition}{conditions} 
Let us call $\cT$ a \emph{$T$-\good} tree if it satisfies the conditions of \cref{cor:smb-goodtree}.
First we construct a tree $\cT_1$ that will satisfy \cref{it:1}. To do that, we simply add to each leaf at $\cT$ a single son, which is a new leaf: we naturally identify each leaf in $\cT_1$ with an actions from $K$, by considering the father of the leaf.
One can see that, with the definition of HST-metric we have not changed the distances: i.e. $\dist_{\cT_1}=\dist_{\cT}$. In particular, we did not change the covering number or the complexity. (Note however, that this change does effect the Algorithm though, as it depends on the tree representation and not directly on the metric.)

The aforementioned change, did however change the depth of the tree and increased it by one. We can repeat this step iteratively until \cref{it:1} is satisfied. To avoid the notation $\cT_1$, we will simply assume that $\cT$ satisfies \cref{it:1}.

Next, we prove the following statement by induction over $H$ the depth of $\cT$:
We assume that for every tree $\cT$ of depth $H-1$ that satisfies \cref{it:1} the statement holds, and prove it for depth $H$.
For $H=1$, since $\dim \le k$ we have that \cref{it:2} holds.

Next, let $\cT_1$ be a tree that we get from $\cT$ by connecting all the leaves to their grandparents (and removing their fathers from the graph).
The first observation is that we have increased the distance between the leaves, so $\distT \le \Delta_{\cT_1}$. We also assume that $\cT$ is not $T$-well behaved, because otherwise the statement obviously holds for $\cT$ with $\cT'=\cT$.

Given that $\dim(\cT)> 2^{-H+1} k$ we next show that $\dim(\cT_1)= \dim(\cT)$.
\ignore{Note that for every $r>2^{-H+1}$, $\covering_r(\cT)=\covering_r(\cT_1)$.
Therefore a necessary condition for $\dim(\cT)\ne \dim(\cT_1)$ is the following:
\begin{align}\label{eq:cond}
2^{-H+1}k
=
2^{-H+1}\covering_{2^{-H+1}}(\cT_1) \ne 2^{-H+1}\covering_{2^{-H+1}}(\cT)
.\end{align}
Indeed, if \cref{eq:cond} is not satisfied, then for every level $r\ge 2^{-H+1}$ we have that $r\covering_r(\cT)=r\covering_r(\cT_1)$, and further, since $2^{-H}\covering_{2^{-H}}(\cT)\le 2^{-H+1}|K|$ we have by assumption that $\dimu(\cT)\ne  2^{-H}\covering_{2^{-H}}(\cT)$. Thus, since the covering number of the tree are equal at all levels $r>2^{-H}$ we have that the \upperlow of $\cT$ and $\cT_1$ is equal. Thus, we've shown that \cref{eq:cond} is satisfied.
}
Note that by construction for every $r=2^{-1},\ldots,2^{-H+2}$ we have that $\covering_{r}(\cT)=\covering_r(\cT_1)$.  We also have by assumption $\dim(\cT)> 2^{-H}k$ and since any covering is smaller than $k$ we also have $\dim(\cT)> 2^{-H+1}\covering_{2^{-H+1}}(\cT)$. Overall, by defintion of $\dim(\cT)$ we have that $\dim(\cT)=\sup_{2^{-H+2} \le \epsilon\le 1} \covering_\epsilon(\cT)$.
Hence,
\begin{align*}
\dimu(\cT_1)&=  \sup_{0<\epsilon\le 1} \epsilon \covering_{\epsilon}(\cT_1)
\\
&= \max\lrset{ \max_{2^{-H+2}<\epsilon\le 1}\epsilon\covering_{\epsilon}(\cT_1) ~,~ 2^{-H+1}k }
 = \max\lrset{ \max_{2^{-H+2}<\epsilon\le 1}\epsilon\covering_{\epsilon}(\cT) ~,~ 2^{-H+1}k }
\\
&=
\max\set{ \dim(\cT) ~,~ 2^{-H+1}k}
\\
&=
\dimu(\cT)
.
\end{align*}
Next, we assume that $\cT_1$ does not satisfy \cref{it:1}.
We then have
$
2^{-(H-1)}  T
>
\sqrt{2^{H-1} \dim(\cT_1) T}
=
\sqrt{2^{H-1} \dim(\cT) T}
,
$
which implies that $\cT$ satisfies \cref{it:3}.
Thus, either $\cT$ is \good or we can construct a tree $\cT_1$ with depth $H-1$ such that $\distT \le \Delta_{\cT_1}$, $\dimu(\cT)=\dimu(\cT_1)$ and $\cT_1$ satisfies \cref{it:1}. The result now follows by an induction step.
\end{proof}



\subsection{Proof of \cref{thm:upper}}

\begin{theorem*}
Let $(K,\dist)$ be a finite metric space over $\abs{K} = k$ elements with diameter $\le 1$ and \upper $\dimu = \dimu(\dist)$.
There exists an algorithm such  that for any sequence of loss functions $\ell_1,\ldots,\ell_T$ guarantees that
$
\mregret(\ell_{1:T},\dist)
=
\tO\Lr{ \max\Lrset{ \dimu^{1/3} T^{2/3}, \sqrt{kT} } }
.
$
\end{theorem*}

\begin{proof}
Given a finite metric space $(K,\Delta)$ we have by \cref{lem:main2} a tree $\cT$ with complexity $\dim=\dimu(\Delta)$ such that $\dist(i,j)\le 4\distT(i,j)$. We can apply SMB as depicted in \cref{lem:tree} over the sequence of losses $\frac{1}{4}\ell_1,\ldots, \frac{1}{4}\ell_T$ To obtain:
\begin{align*}
\frac{1}{4}\regret_T& = \frac{1}{4} \EE{\sum_{t=1}^T \ell_t(i_t) + \frac{1}{4} \dist(i_t,i_{t-1})} - \min_{i^*\in K} \sum \frac{1}{4} \ell_t(i^*)
\\
&\le  \frac{1}{4}\EE{\sum_{t=1}^T \ell_t(i_t) + \distT(i_t,i_{t-1})} - \min_{i^*\in K} \sum \frac{1}{4} \ell_t(i^*)
\\
& = \tO\left(\max\left(\dim^{1/3}T^{2/3},\sqrt{kT}\right)\right)
.&&\qedhere
\end{align*}
\end{proof}


\subsection{Proof of \cref{thm:lower}}\label{sec:lower}

We next set out to prove the lower bound of \cref{thm:lower}.
We begin by recalling the known lower bound for MAB with unit switching cost.

\begin{theorem}[\citet{dekel2014bandits}]\label{thm:koren}
Let $(K,\dist)$ be a metric space over $\abs{K}=k\ge 2$ actions and $\dist(i,j)=c$ for every $i \ne j \in K$.
Then for any algorithm, there exists a sequence $\ell_1,\ldots,\ell_T$ such that
\[
\mregret(\ell_{1:T},\dist)
=
\wt{\Omega}\Lr{(ck)^{1/3}T^{2/3}}.
\]
\end{theorem}

Note that for a discrete metric the minimum covering of $k$ points with balls of radius $c<1$ is by $k$ balls, hence $\covering_{c}(\dist)=k$. Thus we see that \cref{thm:koren} already gives \cref{thm:lower} for the special case of a unit-cost metric (up to logarithmic factors).
The general case can be derived by embedding the lower bound construction in an action set that constitute a $c$-packing of size $\packing_c(\dist)$.

\begin{proof}[Proof of \cref{thm:lower} (sketch)]
First, it is easy to see that the adversary can always force a regret of $\Omega(\sqrt{kT})$; indeed, this lower bound applies for the MAB problem even when there is no movement cost between actions \citep{auer2002nonstochastic}.
We next show a regret lower bound of $\Omega(\diml^{1/3} T^{2/3})$.
By definition, there exist $\eps$ such that $\diml = \eps \packing_{\eps}(\dist)$.
Let $\B_\eps(i_1),\ldots,\B_\eps(i_n)$ be a set of balls that form a maximal packing with $n=\packing_{\eps}(\dist)$, and observe that $\dist(i,i') \ge \eps$ for all $i,i' \in \set{i_1,\ldots,i_n}$, $i \ne i'$. Since we assume the diameter of the metric space is exactly $1$ we have that for all $\epsilon<1$, $\packing_{\eps}(\dist)\ge 2$. Therefore we may assume $n\ge 2$.
We can now use \cref{thm:koren} to show that for any algorithm, one can construct a sequence $\ell_1,\ldots,\ell_T$ of loss functions supported on $i_1,\ldots,i_n$ (and extend them to the entire domain $K$ by assigning a maximal loss of $1$ to any $i \notin \set{i_1,\ldots,i_n}$) such that
\begin{align*}
\mregret(\ell_{1:T},\dist)
=
\Omega\Lr{ (\eps n)^{1/3} T^{2/3} }
=
\Omega\Lr{ \diml^{1/3} T^{2/3} }
.&\qedhere
\end{align*}
\end{proof}

\section{Analysis of SMB for General HSTs}
\label{sec:smb-analysis}

In this section, we extend the analysis given in \cite{koren2017bandits} for the SMB algorithm (\cref{alg:smb}) to general HST metrics over finite action sets,
and prove the following theorem.

\begin{theorem} \label{thm:main}
Assume that the metric $\dist = \distT$ is a metric specified by a tree $\cT$ which is a HST with $\depth(\cT) = H$ and \upperlow $\dim(\cT)=\dim$.
Then, for any sequence of loss functions $\ell_1,\ldots, \ell_T$, \cref{alg:smb} guarantees that
\begin{align*}
\mregret(\ell_{1:t},\distT)
=
\O\LR{ \frac{H\log{\dim}}{\eta} + \eta \dim H 2^H T + H 2^{-H} T }
.
\end{align*}
In particular, by setting $\eta=\Theta\Lr{\sqrt{2^{-H}\log(\dim)/\dim T}}$, the bound on the expected movement regret of the algorithm becomes
\begin{align*}
\mregret(\ell_{1:t},\distT)
=
\O\lr{ H \sqrt{T 2^H \dim\log{\dim}} + H 2^{-H} T }
.
\end{align*}
\end{theorem}


The main new ingredients in the generalized proof are bounds on the bias and the variance of the loss estimates $\tell_t$ used by \cref{alg:smb}, which we give in the following two lemmas.
In the proof of both, we require the following inequality:
\begin{align} \label{eq:biasvar}
\frac{1}{H} \sum_{h=0}^{H-1} \sum_{i \in K} \frac{2^h}{\abs{A_h(i)}}
\le
2^H \dim
.
\end{align}
This follows from the fact that $\sum_{i \in K} \abs{A_h(i)}^{-1}$ equals $\covering(\distT,2^{h-H})$ (both quantities are equal to the number of nodes in the $h$'th level of $\cT$), and since $2^{h-H} \covering(\distT,2^{h-H}) \le \dimu(\distT) = \dim$ by definition of the (covering) \upperlow of $\cT$.

We begin with bounding the bias of the estimator $\tell_t$ from the true loss vector $\ell_t$.

\begin{lemma} \label{lem:unbiased}
For all $t$, we have $\E[\tell_t(i)] \le \ell_t(i)$ and $\E[\ell_t(i_t)] \le\E[p_t \cdot \tell_t] + \eta H 2^H \dim$.
\end{lemma}

\begin{proof}
The proof of the first inequality is identical to the one found in \cite{koren2017bandits} and thus omitted.
%
To bound $\E[\ell_t(i_t)]$, observe that $\E[p_t \cdot \tell_t \mid i_t \in E_t] = 0$ and
\begin{align*}
\E[p_t \cdot \tell_t \mid i_t \notin E_t]
=
\E[p_t \cdot \bell_{t,0} \mid i_t \notin E_t] + \sum_{h=0}^{H-1} \E[\sig_{t,h}] \, \E[p_t \cdot \bell_{t,h} \mid i_t \notin E_t]
=
\E[\ell_t(i_t) \mid i_t \notin E_t]
.
\end{align*}
Then, denoting $\beta_t= \Pr\left[i_t\in E_t\right]$, we have
\begin{align*}
\E[\ell_t(i_t)]
&=
\beta_t \E[\ell_t(i_t) \mid i_t \in E_t] + (1-\beta_t) \E[\ell_t(i_t) \mid i_t \notin E_t]
\\
&\le
\beta_t + (1-\beta_t) \E[p_t \cdot \tell_t \mid i_t \notin E_t]
\\
&=
\beta_t + \E[p_t \cdot \tell_t]
,
\end{align*}
where for the inequality we used the fact that $\ell_t(i_t) \le 1$.

To complete the proof, we have to show that $\beta_t \le \eta H 2^H \dim$.
To this end, write
\begin{align*}
\beta_t
=
\Pr[i_t \in E_t]
\le
\sum_{h=0}^{H-1} \Pr[p_t(A_h(i_t)) < 2^h\eta]
.
\end{align*}
Using \cref{eq:property} to write
\begin{align*}
\EE{ \frac{1}{p_t(A_h(i_t))} }
=
\sum_{i \in K} \frac{1}{\abs{A_h(i)}} \EE{ \frac{\ind{i_t \in A_h(i)}}{p_t(A_h(i))} }
=
\sum_{i \in K} \frac{1}{\abs{A_h(i)}}
,
\end{align*}
together with Markov's inequality, we obtain
\begin{align*}
\Pr\!\big[ p_t(A_h(i_t)) < 2^h \eta \big]
=
\Pr\lrbra{ \frac{1}{p_t(A_h(i_t))} > \frac{1}{2^h \eta} }
\le
\eta \sum_{i \in K} \frac{2^h}{\abs{A_h(i)}}
.
\end{align*}
Using \cref{eq:biasvar}, we conclude that
\begin{align*}
\beta_t
\le
\eta \sum_{h=0}^{H-1} \sum_{i \in K} \frac{2^h}{\abs{A_h(i)}}
\le
\eta H 2^H \dim
.&\qedhere
\end{align*}
\end{proof}

We proceed to control the variance of the estimator $\tell_t$.

\begin{lemma} \label{lem:variance}
For all $t$, we have $\E[p_t \cdot \tell_t^2] \le 2H 2^H \dim$.
\end{lemma}

\begin{proof}
We begin by bounding
\begin{align*}
\tell_t^2(i)
\le
\Lr{ \bell_{t,0}(i) + \sum_{h=0}^{H-1} \sig_{t,h} \bell_{t,h}(i) }^2
.
\end{align*}
Since $\E[\sig_{t,h}] = 0$ and $\E[\sig_{t,h} \sig_{t,h'}] = 0$ for all $h \ne h'$, we have for all $i$ that
\begin{align} \label{eq:var1}
\E[\tell_t^2(i)]
=
\E[\tell_{t,0}^2(i)] +
\sum_{h=0}^{H-1} \E[ \bell_{t,h}^2(i) ]
\le
2\sum_{h=0}^{H-1} \E[ \bell_{t,h}^2(i) ]
.
\end{align}
Following \cite{koren2017bandits}, we have for all $h$ by \cref{lem:bell} that
\begin{align*}
p_t \cdot \bell_{t,h}^2
&\le
\frac{\sum_{i \in K} p_t(i) \ind{i_t \in A_h(i)}}{p_t(A_h(i_t))^2} \prod_{j=0}^{h-1} (1+\sig_{t,j})^2
\\
&=
\frac{1}{p_t(A_h(i_t))} \prod_{j=0}^{h-1} (1+\sig_{t,j})^2
\\
&=
\sum_{i \in K} \frac{1}{\abs{A_h(i)}} \frac{\ind{i_t \in A_h(i)}}{p_t(A_h(i))} \prod_{j=0}^{h-1} (1+\sig_{t,j})^2
.
\end{align*}
Now, since $i_t$ is independent of the $\sig_{t,j}$, and recalling \cref{eq:property}, we get
\begin{align*}
\E_t[p_t \cdot \bell_{t,h}^2]
\le
\sum_{i \in K} \frac{1}{\abs{A_h(i)}} \EE{ \frac{\ind{i_t \in A_h(i)}}{p_t(A_h(i))} } \prod_{j=0}^{h-1} \E[(1+\sig_{t,j})^2]
=
\sum_{i \in K} \frac{2^h}{\abs{A_h(i)}}
.
\end{align*}
This, combined with \cref{eq:var1,eq:biasvar}, gives the result:
\begin{align*}
\E[p_t \cdot \tell_t^2]
\le
2\sum_{h=0}^{H-1} \E[p_t \cdot \bell_{t,h}^2]
\le
2 \sum_{h=0}^{H-1} \sum_{i \in K} \frac{2^h}{\abs{A_h(i)}}
\le
2H 2^H \dim
&.\qedhere
\end{align*}
\end{proof}

\subsection{Additional Lemmas from \cite{koren2017bandits}}

We state several lemmas proved in \cite{koren2017bandits} that are required for our generalized analysis; we refer to the original paper for the proofs.



\begin{lemma} \label{lem:bell}
For all $t$ and $0 \le h < H$ the following holds almost surely:
\begin{align} \label{eq:bell1}
0 \le \bell_{t,h}(i) \le \frac{\ind{i_t \in A_h(i)}}{p_t(A_h(i))}
\prod_{j=0}^{h-1} (1+\sig_{t,j}) \qquad \forall ~ i \in K \,.
\end{align}
In particular, if $\sig_{t,j} = -1$ then $\bell_{t,h} = 0$ for all $h > j$.
As a result,
\begin{align} \label{eq:tell-equiv}
\tell_t
=
\bell_{t,0} - \bell_{t,h_t} + \sum_{j=0}^{h_t-1} \bell_{t,j}
.
\end{align}
\end{lemma}

\begin{lemma} \label{lem:sampling}
For all $t$ and $0 \le h < H$ the following hold:
\begin{enumerate}[label=(\roman*)]
\item for all $A \in \set{ A_h(i) : i \in K}$ we have
\begin{align} \label{eq:property}
\EE{\frac{\ind{i_t\in A}}{p_t(A)}}=1 ~;
\end{align}
\item with probability at least $1-2^{-(h+1)}$, we have that $A_h(i_t)=A_h(i_{t-1})$.
\end{enumerate}
\end{lemma}

\begin{lemma}[Second-order regret bound for MW] \label{lem:mw2}
Let $\eta > 0$ and let $c_1,\ldots,c_T \in \reals^k$ be real vectors such that $c_t(i) \ge -1/\eta$ for all $t$ and $i$.
Consider a sequence of probability vectors $q_1,\ldots,q_T$ defined by $q_1 = (\tfrac{1}{k},\ldots,\tfrac{1}{k})$, and for all $t > 1$:
\begin{align*}
q_{t+1}(i) = \frac{ q_t(i) \, e^{-\eta c_t(i)} }{ \sum_{j=1}^k
q_t(j) \, e^{-\eta c_t(j)} } \qquad \forall ~ i \in [k] .
\end{align*}
Then, for all $i^\star \in [k]$ we have that
\begin{align*}
\sum_{t=1}^T q_t \cdot c_t - \sum_{t=1}^T c_t(i^\star)
\le
\frac{\ln{k}}{\eta} + \eta \sum_{t=1}^T q_t \cdot c_t^2
.
\end{align*}
\end{lemma}


\subsection{Regret Analysis}

We now have all we need in order to prove our main result.


\begin{proof}[Proof of \cref{thm:main}]
First, we bound the expected movement cost.
\cref{lem:sampling} says that with probability at least $1-2^{-(h+1)}$, the actions $i_t$ and $i_{t-1}$ belong to the same subtree on level $h$ of the tree, which means that $\distT(i_t,i_{t-1}) \le 2^{h-H}$ with the same probability.
Hence,
\begin{align} \label{eq:Emove}
\E[\distT(i_t,i_{t-1})]
\le
\sum_{h=0}^{H-1} 2^{h-H} \Pr\!\big[ \distT(i_t,i_{t-1}) > 2^{h-H} \big]
\le
\sum_{h=0}^{H-1} 2^{-(H+1)}
=
\frac{H}{2^{H+1}}
,
\end{align}
and the cumulative movement cost is then $\O(H 2^{-H} T)$.

We turn to analyze the cumulative loss of the algorithm.
We begin by observing that $\tell_t(i) \ge -1/\eta$ for all $t$ and $i$.
To see this, notice that $\tell_t = 0$ unless $i_t \notin E_t$, in which case we have, by \cref{lem:bell} and the definition of $E_t$,
\begin{align*}
0
\le
\bell_{t,h}(i)
\le
\frac{2^h}{p_t(A_h(i_t))}
\le
\frac{1}{\eta}
\qquad\quad
\forall ~ 0 \le h < H
,
\end{align*}
and since $\tell_t$ has the form $\tell_t = \bell_{t,0} + \sum_{j=0}^{h_t-1} \bell_{t,j} - \bell_{t,h_t}$ (recall \cref{eq:tell-equiv}), we see that $\tell_t(i) \ge -1/\eta$.
Hence, we can use second-order bound of \cref{lem:mw2} on the vectors $\tell_t$ to obtain
\begin{align*} 
\sum_{t=1}^T p_t \cdot \tell_t - \sum_{t=1}^T \tell_t(i^\star)
\le
\frac{\ln{k}}{\eta} + \eta \sum_{t=1}^T p_t \cdot \tell_t^2
\end{align*}
for any fixed $i^\star \in K$.
Taking expectations and using \cref{lem:unbiased,lem:variance},
and using the rough bound $k \le 2^H \dim$, we have
\begin{align} \label{eq:Eregret}
\EE{ \sum_{t=1}^T \ell_t(i_t) } - \sum_{t=1}^T \ell_t(i^*)
\le
\frac{H\ln(\dim)}{\eta} + 3\eta H 2^H T \dim
.
\end{align}
The theorem now follows from \cref{eq:Emove,eq:Eregret}.
\end{proof}

\end{document}